\documentclass{article}
\usepackage{kr}

\usepackage{times}
\usepackage{soul}
\usepackage{url}
\usepackage[hidelinks]{hyperref}
\usepackage[utf8]{inputenc}
\usepackage[small]{caption}
\usepackage{graphicx}
\usepackage{amsmath}
\usepackage{amsthm}
\usepackage{booktabs}
\usepackage{algorithm}
\usepackage{algorithmic}
\usepackage[dvipsnames]{xcolor}

\newtheorem{example}{Example}
\newtheorem{theorem}{Theorem}

\newcommand{\Fr}[1]{\mathrm{fr}(#1)}
\newcommand{\Const}[1]{\mathrm{const}(#1)}
\newcommand{\Term}[1]{\mathrm{term}(#1)}
\newcommand{\Var}[1]{\mathrm{var}(#1)}
\newcommand{\head}[1]{\mathrm{head(#1)}}
\newcommand{\body}[1]{\mathrm{body(#1)}}

\newcommand{\sep}[2]{\mathrm{sep}_{#2}(#1)}

\newcommand{\ruleset}{\mathcal R}
\newcommand{\instance}{I}

\newcommand{\chase}[3]{\mathrm{chase}_{#1}(#2,#3)}

\newtheorem{definition}{Definition}
\newtheorem{corollary}{Corollary}
\newtheorem{proposition}{Proposition}

\renewcommand{\phi}{\varphi}

\newcommand{\mlm}[1]{} 
\newcommand{\mlmm}[1]{\textcolor{black}{#1}} 

 \newcommand{\vect}[1]{\mathbf{#1}}

\title{Parallelisable Existential Rules: a Story of Pieces}

\author{
Maxime~Buron$^1$\and
Marie-Laure~Mugnier$^2$\and
Micha\"{e}l~Thomazo$^3$
\affiliations
$^1$ University of Oxford, United Kingdom \\
$^2$ LIRMM, Inria, University of Montpellier, CNRS, France \\
$^3$ Inria, DI ENS, ENS, CNRS, PSL University \& Inria, France \\
}

\begin{document}

\maketitle

\begin{abstract}
In this paper, we consider existential rules, an expressive formalism well suited to the representation of ontological knowledge and data-to-ontology mappings in the context of ontology-based data integration. 
The chase is a fundamental tool to do reasoning with existential rules as it computes all the facts entailed by the rules from a database instance. We introduce parallelisable sets of existential rules, for which the chase can be computed in a single breadth-first step from any instance. The question we investigate is the characterization of such rule sets. We show that parallelisable rule sets are exactly those rule sets both bounded for the chase and belonging to a novel class of rules, called pieceful. The pieceful class includes in particular frontier-guarded existential rules and (plain) datalog. We also give another characterization of parallelisable rule sets in terms of rule composition based on rewriting. 

\end{abstract}

\section{Introduction}


Ontology-based data access (OBDA) systems aim at facilitating data querying through a conceptual layer formalized by an ontology \cite{jods-08-plcglr,DBLP:conf/ijcai/XiaoCKLPRZ18}. They rely on a three-level architecture comprising  the ontology, the data sources and the mapping between the two. The key idea is that a user expresses queries at a conceptual level, and the system translates these queries into queries on the data 
via the mapping, while integrating ontological reasoning.  

When we abstract away from data sources and mappings, we obtain the fundamental \emph{ontology-based query answering} problem, which takes as input an ontology $\mathcal{O}$, an  
instance (or set of facts) $\instance$ and a (Boolean) conjunctive query $q$, both expressed in the vocabulary of $\mathcal O$, and asks whether $\instance,\mathcal{O}\models q$. Ontological knowledge is typically represented in \emph{description logics} (\emph{e.g.}, \cite{DBLP:books/daglib/0041477,DBLP:conf/rweb/BienvenuO15}) or \emph{existential rules} (\emph{e.g.}, \cite{Cali2009,blms:09}) and we shall consider the latter language in this paper. 
Existential rules are an extension of first-order function-free Horn rules allowing for
existentially quantified variables in the rule heads (\emph{e.g.}, $\forall x (s(x) \rightarrow \exists z~t(x,z))$), which makes them able to infer the existence of unknown individuals. They generalise datalog and most description logics used to do reasoning on data, namely Horn description logics. 

Two dual techniques are used to solve the ontology-based query answering problem: the \emph{chase}, which enriches $\instance$ by performing a fixpoint computation with $\mathcal{O}$ until a canonical model of $\instance$ and $\mathcal{O}$ is obtained (then $q$ is asked on the result of the chase), and \emph{query rewriting}, where $q$ is rewritten with $\mathcal{O}$ into a query $q'$, such that for all instance $\instance$, holds $\instance, \mathcal{O} \models q$ if and only if holds $\instance \models q'$. 
Query answering is undecidable with general existential rules, however there are expressive subclasses ensuring the termination of either technique.

In the OBDA paradigm, the instance $\instance$ is not materialized, but virtually defined by the mapping and the data.  
Precisely, an OBDA specification is given by an ontology $\mathcal O$, a relational schema $\mathcal S$ and a mapping $\mathcal M$  from  $\mathcal S$ to $\mathcal O$  \cite{DBLP:journals/aim/Lenzerini18}.
The mapping is itself composed of assertions of the form
$q_S(\vect{x}) \rightarrow q_O(\vect{x})$, where $q_S$ is a query on schema $\mathcal S$ and $q_O$ is a conjunctive query on the vocabulary of $\mathcal O$, 
both with tuple of answer variables $\vect{x}$. 
When $q_S$ is also a conjunctive query, or a relational view, a mapping assertion can be seen as an existential rule. 
Then, the virtual instance $\instance_{(D,\mathcal M)}$, associated with a database $D$ (on $\mathcal S$) and the mapping $\mathcal M$, is the set of facts that would be obtained by chasing $D$ with $\mathcal M$. Note that only a single (breadth-first) step of the chase is required here, as bodies and heads of mapping assertions are on disjoint sets of predicates. 
Since $\instance_{(D,\mathcal M)}$ is virtual, an incoming query has to be rewritten, first with $\mathcal O$, then with $\mathcal M$, which yields a query directly asked on $D$. 
As rewriting is performed at query time, speeding up this process is a crucial issue. In particular, query rewriting with $\mathcal O$ is a recursive process, which is not the case with $\mathcal M$. 
For very lightweight ontology languages (the DL-Lite family or the W3C language RDFS), a practically efficient approach consists of compiling (part of) the ontological reasoning into the mapping, so that the rewriting step with $\mathcal O$ can be avoided or drastically reduced \cite{DBLP:conf/semweb/KontchakovRRXZ14,DBLP:conf/edbt/BuronGMM20}. In these settings, each mapping assertion can be processed independently, the mapping head being enriched with  knowledge it entails. 
Whether such technique can be extended to more expressive languages is an open issue, which motivated the work presented here.


Consider an OBDA setting where existential rules are used as a uniform language to express both the ontology and the mapping.
Compiling ontological reasoning into the mapping can be seen as computing a new mapping $\mathcal M'$ such that query rewriting with $\mathcal O$ and $\mathcal M$ is reduced to query rewriting with $\mathcal M'$. From a dual viewpoint, for any database $D$, the instance $\instance_{(D,\mathcal M')}$ is equivalent to the chase of $\instance_{(D,\mathcal M)}$ with $\mathcal O$.
The next example illustrates the approach. 

\begin{example}
Let $\mathcal M = \{M_1,M_2\}$ and $\mathcal O = \{R_1, R_2\}$, where $s_i$ and $t_i$ denote predicates from $\mathcal S$ and $\mathcal O$, respectively, and universal quantifiers are omitted: \\
$M_1 = s_1(x,y) \rightarrow t_1(x,y)$\\
$M_2 = s_2(x)  \rightarrow t_2(x)$\\
$R_1 = t_2(x)  \rightarrow \exists z~t_3(x,z)$\\
$R_2 = t_1(x,y) \wedge t_3(x,z)  \rightarrow t_4(y)$
\\
Here, the ontology can be compiled into the mapping, which yields $\mathcal M' = \mathcal M \cup \{M_3,M_4\}$, where:
\\
$M_3 = s_2(x) \rightarrow \exists z ~t_3(x, z)$\\
$M_4 = s_1(x,y) \wedge s_2(x) \rightarrow t_4(y)$. \\
 Intuitively, $\mathcal M'$ is obtained by composing the rules from $\mathcal M \cup \mathcal O$ until a fixpoint is reached (see Sect. \ref{sec-composition}), then keeping only mapping assertions, \emph{i.e.,} rules whose all body predicates are in  $\mathcal S$. 
 We can check that, for any database $D$ on $\mathcal S$, a single breadth-first step of the chase of $D$ with $\mathcal M'$ suffices to produce the chase of $D$ with $\mathcal M \cup \mathcal O$. 
\end{example}
 
With the aim of developing compilation techniques for OBDA systems based on existential rules, we asked ourselves the following question: under which conditions on the rules can the chase be simulated in a single step?

%
%

We formalize the desired property by the notion of \emph{parallelisable} existential rule sets. Informally, a (finite) rule set $\ruleset$ is parallelisable if there exists a finite rule set $\ruleset'$ able to produce (an equivalent superset of) 
the chase of $\ruleset$ in a \emph{single} breadth-first step, independently from any instance. 
The question we  investigate in this paper is how to characterize such rule sets.   Our main results are the following. 
\begin{itemize}
\item Clearly, \emph{boundedness}, which expresses that the number of chase steps is bounded independently from any instance, is a necessary condition for parallelisability. This notion has long been studied for datalog \cite{DBLP:journals/jlp/HillebrandKMV95} and more recently for existential rules \cite{DBLP:conf/ijcai/BourhisLMTUG19,DBLP:journals/tplp/DelivoriasLMU21}. While boundedness is equivalent to parallelisability in the case of datalog, it does not ensure the parallelisability of an existential rule set. This leads us to define a new class of existential rules, namely \emph{pieceful}. 
We show that parallelisable rule sets are exactly those sets that are both bounded and pieceful.
\item The novel pieceful class has an interest in itself, as it generalizes datalog as well as a main class of existential rules, namely frontier-guarded, which itself covers some prominent description logics (see the last section for details). 
 \item Piecefulness is based on the behavior of rules during the chase. Adopting the viewpoint of query rewriting, we study an operator of rule composition (that we call \emph{existential composition}) based on so-called piece-unifiers, a notion at the core of rewriting with existential rules. We show that any bounded and pieceful rule set can be parallelized by a finite set of composed rules.  
 \item This allows us to provide another characterization of pieceful rule sets based on their behavior during rule composition. Roughly,  a certain property (that we call \emph{existential stability}) has to be fulfilled by every pair of rules of the set and preserved whenever composed rules are added to the set. 
 \item Finally, noting that rule composition does not fit well with the behavior we intuitively expected, we introduce another rule composition operator (that we call \emph{compact composition}). The result of this operator goes beyond the existential rule language. We show however that, when the stability property is satisfied, both kinds of compositions yield logically equivalent formulas. 
\end{itemize}

We believe that these results open up many practical and theoretical perspectives, which we discuss in the last section. 
\mlm{Add here: Full proofs that could not be included due to space restrictions are available in a technical report [REF]?}


%
%
\section{Preliminaries}\label{sec-prelim}
%
%
\newcommand{\exist}[1]{\mathrm{exist}(#1)}

\paragraph{Generalities.}
A vocabulary is a pair $\mathcal V = (\mathcal P, \mathcal C)$, where $\mathcal P$ is a finite set of predicates and  $\mathcal C$ is a possibly infinite set of constants.  A \emph{term} on $\mathcal V$ is a constant from $\mathcal C$ or a variable. An
\emph{atom} on $\mathcal V$ has the form $p(\vect{t})$ where $p \in \mathcal P $ is a predicate of arity $n$ and $\vect{t}$ is a tuple of terms on $\mathcal V$ with $|\vect{t}|=n$. 
An atom is \emph{ground} if it has no variable.  
Given an atom or set of atoms $S$, we denote by $\Var{S}$, $\Const{S}$ and $\Term{S}$ its sets of variables, constants and terms, respectively. 
We will often see a tuple $\vect{x}$ of pairwise distinct variables as a set.  
We denote by $\models$ the classical logical consequence. 
Given two sets of atoms $S_1$ and $S_2$,
a \emph{homomorphism} $h$ from $S_1$ to $S_2$ is a substitution of $\Var{S_1}$
by $\Term{S_2}$ such that $h(S_1) \subseteq S_2$ (we say that $S_1$ \emph{maps} to 
 $S_2$ by $h$).

\vspace*{-0.4cm}
\paragraph{Instances and Rules.} An \emph{instance} is a finite set of ground atoms. Any finite set of atoms $S$ can be turned into an instance,
denoted by $\mathrm{freeze}(S)$, by \emph{freezing} its variables, \emph{i.e.}, bijectively renaming each variable by a fresh constant. 
An \emph{extended instance} is a (possibly infinite) set of atoms, in which variables are classically called \emph{nulls}. The associated (possibly infinite) formula
is the existential closure of the conjunction of the atoms. 
 An \emph{existential rule} $R$ (or simply \emph{rule} hereafter) is a closed formula of  the form  
  $$\forall\vect{x}\forall\vect{y} ~[~B(\vect{x},\vect{y}) \rightarrow \exists \vect{z} ~H(\vect{x},\vect{z})~]$$
 where $B$ and $H$ are non-empty and finite conjunctions of atoms on variables, respectively called the \emph{body} and  the \emph{head} of the rule, denoted by $\body{R}$ and $\head{R}$, and $\vect{x}, \vect{y}$ and $\vect{z}$ are pairwise disjoint. We make the common assumption that rules do not contain constants, which 
 simplifies technical tools. 
 The set $\vect x$ is called the \emph{frontier} of $R$ and is denoted by $\Fr{R}$. Its elements are called frontier variables. 
 The set $\vect{z}$ is called the set of \emph{existential variables} (of $R$) and is denoted by $\exist{R}$.
An existential rule $R$ is \emph{datalog} if $|\head{R}| = 1$ and $\exist{R} = \emptyset$.
  For brevity, we often denote by $B \rightarrow H$ a rule with body $B$ and head $H$. In the following, we denote by $\mathcal R$ a finite set of existential rules and assume w.l.o.g. that distinct rules in $\mathcal R$ have disjoint sets of variables. In our examples, we reuse variables for simplicity. 
%
 %
%


A rule $R = B \rightarrow H$  is \emph{applicable} to a set of atoms $S$ if there is a homomorphism $\pi$ from $B$ to $S$. The pair $(R,\pi)$ is called a \emph{trigger} for $S$. The application of $R$ according to $\pi$ (or: of the trigger $(R,\pi)$) produces a set of atoms obtained from $\head{R}$ by replacing each frontier variable $x$ with $\pi(x)$ and each existential variable with a fresh variable, called a \emph{null}. We denote by $\pi^{\mathrm{safe}}$ this extension of $\pi$ that ``safely'' renames existential variables; hence, atoms produced by distinct applications of the same rule have disjoint sets of nulls. 
The set of atoms resulting from the application of $(R,\pi)$ to $S$ is $\alpha(S,R,\pi) = S \cup \pi^{\mathrm{safe}}(H)$.

An \emph{$\ruleset$-derivation} (from $\instance$ to $\instance_k$) is a finite sequence $(\instance_0=\instance),(R_1, \pi_1, \instance_1), \ldots, (R_{k}, \pi_{k}, \instance_k)$
such that  for all $0 < i \leq k$, $R_i \in \ruleset$, $(R_i, \pi_i)$ is a trigger for $\instance_{i-1}$  and $\instance_{i} = \alpha(\instance_{i-1},R_i, \pi_i)$. When only the successive (extended) instances are needed, we note $(\instance_0=\instance),\instance_1, \ldots, \instance_k$.

\vspace*{-0.4cm}
\paragraph{Chase.} The chase 
builds a derivation from an instance by repeatedly applying rules until a fixpoint is reached. 
We rely on the \emph{semi-oblivious} chase variant (in short, \emph{so-chase}), in which two triggers $(R,\pi_1)$ and $(R,\pi_2)$ that coincide on the frontier of $R$ produce exactly the same result \cite{DBLP:conf/pods/Marnette09}. More precisely, we assume that nulls are named as follows: given a trigger $(R, \pi)$, for all $z \in \exist{R}$, $\pi^{\mathrm{safe}}(z) = z_{(R, \pi_{|\Fr{R}})}$, where $\pi_{|\Fr{R}}$ denotes the restriction of  $\pi$ to the domain $\Fr{R}$. Hence, the name of a null created by a trigger $(R,\pi)$ is based on $R$ and $\pi_{|\Fr{R}}$, not $\pi$ itself. 
 We consider a breadth-first so-chase defined as follows: $\chase{0}{\instance}{\ruleset} = \instance$ and for any $i > 0$, 
 $$\chase{i}{\instance}{\ruleset} = \chase{i-1}{\instance}{\ruleset} \cup \bigcup_{(R,\pi)}( \pi^{\mathrm{safe}}(\head{R}))$$
where $(R,\pi)$ is any trigger for $\chase{i-1}{\instance}{\ruleset}$ and   for any $z \in \exist{R}$, $\pi^{\mathrm{safe}}(z) = z_{(R, \pi_{|\Fr{R}})}$.
Finally, $\chase{\infty}{\instance}{ \ruleset} = \bigcup_{i \geq 0} \chase{i}{\instance}{ \ruleset}$.  
\begin{example} Let $\instance = \{ p(a,b) \}$ and $\ruleset = \{p(x,y) \rightarrow \exists z ~p(x,z) \land A(z)\}$. By the trigger $t_1 = (R, \{x \mapsto a, y \mapsto b\})$, we obtain
$\chase{1}{\instance}{\ruleset} = \{ p(a,b), p(a, \nu), A(\nu) \}$ with $\nu = z_{(R, \{ x \mapsto a \})}$. The trigger $t_2 = (R, \{x \mapsto a, y \mapsto \nu\})$ then produces the same atoms as  $t_1$. Finally, $\chase{\infty}{\instance}{ \ruleset} = \chase{1}{\instance}{\ruleset}$. 
\end{example}

\vspace*{-0.4cm}
\paragraph{Query Answering.} A \emph{conjunctive query} (CQ) is of the form $q(\mathbf{x}) = \exists \mathbf{y}~\phi(\mathbf{x},\mathbf{y})$, where  $\mathbf{x}$ and $\mathbf{y}$ are disjoint tuples of variables, $\phi$ is a conjunction of atoms, and $\mathbf{x} \cup \mathbf{y} = \textit{var}(\phi)$; the free variables in $\phi$ (\emph{i.e.}, $\mathbf x$) are called answer variables. A \emph{Boolean CQ} has no free variables. A \emph{union of conjunctive queries} (UCQ) is a disjunction of CQs that have the same arity  $|\mathbf{x}|$.  An (extended) instance $\instance$ answers positively to a Boolean CQ $q$ iff $\instance \models q$. More generally, a tuple of constants $\vect{c}$  is an \emph{answer} to a CQ $q(\vect{x})$,
 with $|\vect{x}| = |\vect{c}|$, on $\instance$ if  $\instance \models q[\vect{c}]$, where  $q[\vect{c}]$ is obtained by substituting the i-th variable in $\vect{x}$ by the i-th constant in $\vect{c}$.  
The \emph{CQ answering} problem takes as input an instance $\instance$, a rule set $\ruleset$, a query $q(\vect{x})$ and a tuple of constants $(\vect{c})$, with $|\vect{x}| = |\vect{c}|$,   and asks whether  $\instance, \ruleset \models q[\vect{c}]$. 
It holds that $\instance, \ruleset \models q[\vect{c}]$ iff there is an $\ruleset$-derivation from $\instance$ to $\instance'$ such that  $\instance' \models q[\vect{c}]$. Equivalently, $\instance, \ruleset \models q[\vect{c}]$ iff there is $k$ such that chase$_k(\instance, \ruleset) \models q[\vect{c}]$.



\vspace*{-0.4cm}
\paragraph{Pieces.} The notion of piece  is key in this paper. 
Given a set of atoms $S$, a \emph{piece} of $S$ with respect to a set of terms $T$  is a non-empty set $S' \subseteq S$ such that (1) for any atom $a \in S$, if $a$ shares a term from $T$  with some $a' \in S'$ then $a \in S'$, and (2) there is no strict subset of $S'$ that satisfies  (1), \emph{i.e.}, $S'$ is minimal. Intuitively, atoms of $S$ are glued together by the terms of $T$, which yields pieces.   Note that, for any set of atoms $S$ and set of terms $T$, $S$ can be partitioned into pieces w.r.t. $T$. Here, pieces are defined with respect to existential variables or to nulls, depending on the considered objects. By default, a piece of $S$ is 
w.r.t. $\exist{S}$ if $S$ is a rule head and w.r.t. the nulls in $S$ if $S$ is an (extended) instance. 

A rule is called \emph{single-piece} if its head forms a single piece. 
Any rule can be decomposed  into a (trivially) equivalent set of single-piece rules. For instance, a rule $r(x,y) \rightarrow \exists z_1 \exists z_2~p(x,z_1) \land A(z_1) \land A(z_2) \land p(x,y)$ has a head with three pieces: $\{p(x,z_1), A(z_1)\}$, $\{A(z_2)\}$ and $\{p(x,y)\}$, hence can be decomposed into three single-piece rules: $r(x,y) \rightarrow \exists z_1~p(x,z_1) \land A(z_1)$; $r(x,y) \rightarrow \exists z_2~A(z_2)$ and $r(x,y) \rightarrow p(x,y)$.
In the following, we assume that rules are single-piece. This assumption simplifies our setting, although all notions and results of this paper could be reformulated without making  it. 

\vspace*{-0.3cm}
\paragraph{Piece-Unifiers.} Piece-unifiers are a generalization of classical unifiers that take care of existential variables in rule heads by unifying sets of atoms instead of single atoms
 \cite{blms:09}.  Query rewriting as well as rule composition (see Sect. \ref{sec-composition}) are based on this notion.  In the definition below, we give a simplified version of piece-unifiers, which does not take constants into account. 
See, \emph{e.g.}, \cite{klmt:15} for details about piece-unifiers. 

Given sets of atoms $S$ and $S' \subseteq S$, the set of \emph{separating variables} in $S'$ w.r.t. $S$, denoted by $\sep{S'}{S}$, is the set of variables that belong to both $S'$ and $S\setminus S'$. 
Note that when $S'$ is a piece of an extended instance $\instance$, $\sep{S'}{\instance}$ is necessarily empty (as $S'$ shares only constants with the rest of $\instance$).  

\begin{definition}[Piece-unifier]
  Let $S$ be a set of atoms and $R = B \rightarrow H$ be a rule (both without constants) \mlmm{such that $\Var{S} \cap \Var{B \cup H} = \emptyset$}. A \emph{piece-unifier} of $S$ with $R$ (or with $H$)   
  is a triple $\mu = (S',H', u)$ with $S' \neq \emptyset$, $S' \subseteq S$,  $H' \subseteq H$, and 
  $u$ is a substitution of $\Fr R \cup \Var{S'}$ by $\Var{\head{R}}$ such that:
  \begin{enumerate}
  \item For all $x \in \Fr R$, $u(x) \in \Fr R$
  \item For all $x \in \sep{S'}{S}$, $u(x) \in \Fr R$ 
  \item $u(S') = u(H')$. 
  \end{enumerate}
\end{definition}
%

%

\mlmm{Let $S$ be a set of atoms and $\mu = (S',H', u)$ be a piece-unifier of $S$ with $R: B \rightarrow H$.  
 The \emph{(direct) rewriting} of $S$ w.r.t. $\mu$ is $\beta(S,R,\mu) = u(B) \cup u(S \setminus S')$. An \emph{$\ruleset$-rewriting} $S_k$ of $S$ is obtained by a finite sequence $(S_0=S), \ldots, S_k$ such that, for all $0 < i \leq k$, $S_i$ is a direct rewriting of $S_{i-1}$ w.r.t. a piece-unifier of $S_{i-1}$ with a (copy of a) rule from $\ruleset$. }



\begin{example}[Piece-Unifier] Let $R = A(x) \rightarrow \exists z~p(x,z)$ and $S_1 = \{p(w,v), B(v)\}$. There is no piece-unifier of $S_1$ with $R$ since $v$ is a separating variable of $S_1' = \{p(w,v)\}$, hence cannot be unified with $z$. Let $S_2 = \{ p(w_1, v), B(w_1), p(w_2,v), C(w_2) \}$. 
The triple $\mu=(S_2',H',u)$ with $S_2'= \{p(w_1,v),p(w_2,v)\}$, $H'=\{p(x,z)\}$ and $u = \{w_1, w_2 \mapsto x, ~v \mapsto z\}$  is a piece-unifier of $S_2$ with $R$,  which yields the direct rewriting $\{A(x),B(x),C(x)\}$. 
\end{example}

A fundamental property of piece-unifiers is the following: given any instance $\instance$, rule set $\ruleset$ and Boolean CQ $q$,  there is an $\ruleset$-derivation from $\instance$ to an $\instance_i$ such that $q$ maps to $\instance_i$ ($\instance_i \models q$) iff there is an $\ruleset$-rewriting $q'$ of $q$ such that $q'$ maps to $\instance$ ($\instance \models q'$). Note that this property does not rely on rules being single-piece. 

\vspace*{-0.4cm}
\paragraph{Fundamental Properties of Rule Sets.}
A rule set $\ruleset$ is \emph{so-chase finite} if for any instance $I$, there is $k$ such that $\chase{k}{\instance}{ \ruleset} = \chase{\infty}{\instance}{\ruleset}$.
A rule set $\ruleset$ is a \emph{finite unification set} (fus) if for any Boolean CQ $q$, there is a finite set $Q$ of $\ruleset$-rewritings of $q$, such that, 
for any $\ruleset$-rewriting $q'$ of $q$, there is $q'' \in Q$ that maps to $q'$ \cite{blms:09}. It is not hard to see that fus is equivalent to first-order rewritability (introduced in \cite{dl-lite}):  
$\ruleset$ is first-order rewritable if for any Boolean CQ $q$, there is a UCQ $Q$, such that for any instance $\instance$, it holds that 
$\instance, \ruleset \models q$ iff $\instance \models Q$. 
Moreover, as several times remarked, fus is equivalent to the \emph{bounded derivation depth property} \cite{Cali2009}: for any Boolean CQ $q$, there is $k$ such that for any instance $\instance$, if $\instance, \ruleset \models q$ then $\chase{k}{\instance}{ \ruleset} \models q$. 
Finally, a rule set is \emph{so-bounded} if there is $k$ (a bound) such that for any instance $\instance$, $\chase{k}{\instance}{ \ruleset} = \chase{\infty}{\instance}{\ruleset}$
\cite{DBLP:journals/tplp/DelivoriasLMU21}. 
Clearly, when a rule set is so-bounded, it is also so-chase finite and fus.
The reciprocal holds true for single-piece rules (follows from \cite{DBLP:conf/ijcai/BourhisLMTUG19}).
In the remaining, we will simply write chase and bounded in place of so-chase and so-bounded.

\begin{example} $~$\\
$\ruleset_1 = \{A(x) \rightarrow \exists z~p(x,z) \land A(z)\}$ is fus and not chase-finite. 
\\
$\ruleset_2 = \{R_1 : p(x,y) \wedge p(y,z) \rightarrow p(x,z)\}$ is datalog (hence chase-finite) and not fus (hence not bounded). 
\\
$\ruleset_3 = \ruleset_2 \cup \{ R_2 : p(x,y) \wedge p(u,z) \rightarrow p(x,z) \}$ is bounded. Note that the body of $R_2$  contains `disconnected' atoms and all the atoms produced by $R_1$ are also produced by $R_2$; moreover, for any instance $\instance$, all the atoms producible by $R_2$ are produced at the first breadth-first step of the chase. That is why $\ruleset_3$ is bounded with bound $1$.    
\end{example}



%


\section{Characterization of parallelisable rule sets}\label{sec-parallel}



In this section, we define parallelisable rule sets and motivate the introduction of a new class of existential rules, namely pieceful. We show that parallelisable rule sets are exactly those rule sets that are both bounded and pieceful.

\subsection{Parallelisable Rule Sets}
Parallelisability intuitively means that a sound superset of the chase can be obtained in a single breadth-first step by a finite set of rules independent from any instance. 

\begin{definition}[Parallelisability]
 A set of rules $\ruleset$ is \emph{parallelisable} if there exists a finite rule set $\ruleset'$ such that for any instance $\instance$:
 \begin{enumerate}
  \item there is an injective homomorphism from $\chase{\infty}{\instance}{\ruleset}$ to $\chase{1}{\instance}{\ruleset'}$; 
  \item there is a homomorphism from $\chase{1}{\instance}{\ruleset'}$ to $\chase{\infty}{\instance}{\ruleset}$.
 \end{enumerate}
 Such rule set $\ruleset'$ is said to \emph{parallelise} $\ruleset$.  
\end{definition}

\vspace*{-0.5cm}
\paragraph{Note on this definition.} A more powerful notion of parallelisability could be obtained by dropping the injectivity requirement in Point 1. Then, $\chase{1}{\instance}{\ruleset'}$ would 
 be equivalent to $\chase{\infty}{\instance}{\ruleset}$ but would not necessarily include it. 
This notion could also be obtained without changing the definition but considering a stronger chase variant for $\chase{\infty}{\instance}{\ruleset}$, namely the core chase, which produces a minimal canonical model of $\instance \cup \ruleset$ \cite{deutsch:08}.  However, the core chase is not monotonic and its result may even not be obtainable by any $\ruleset$-derivation. We have chosen here to consider the well-behaved so-chase. 

\begin{proposition}\label{prop-parallelisable-imply-bounded}
 If $\ruleset$ is parallelisable, then it is bounded.
\end{proposition}

\begin{proof}
 Let us assume that $\ruleset$ is parallelised by $\ruleset'$. For any instance $\instance$, $\chase{1}{\instance}{\ruleset'}$ is finite because $\ruleset'$ is finite. As there is an injection from  $\chase{\infty}{\instance}{\ruleset}$ to $\chase{1}{\instance}{\ruleset'}$, $\chase{\infty}{\instance}{\ruleset}$ is finite.
 Now, for each rule $R \in \ruleset'$, let $\instance = \mathrm{freeze}(\body{R})$: since $R$ is applicable to $\instance$ and  $\chase{1}{\instance}{\ruleset'}$ maps to $\chase{\infty}{\instance}{\ruleset}$, there is $k_R$ such that $\head{R}$ maps to $\chase{k_R}{\instance}{\ruleset}$. Let $k$ be the maximum $k_R$ over all rules $R \in \ruleset$.  For any instance $\instance$, there is $n$ such that $\chase{\infty}{\instance}{\ruleset} = \chase{n}{\instance}{\ruleset}$ with $n \leq k$. Hence, $\ruleset$ is bounded (by $k$). 
\end{proof}

The converse is however not true, as witnessed by the following example.

\begin{example}[Prime Example]\label{ex-prime}
Let $\ruleset = \{R_1, R_2\}$ where:
 
 \vspace*{-0.5cm}
\begin{quote}
\item  $R_1: A(x) \rightarrow \exists z~p(x,z)$
\item $R_2 : p(x,z) \wedge B(y) \rightarrow r(z,y)$
\end{quote}
$\ruleset$ is bounded but not parallelisable. Let, for any $n$, $\instance_n = \{A(a),B(b_1),\ldots,B(b_n)\}$. There is a null in $\chase{\infty}{\instance_n}{\ruleset}$ that appears in $n+1$ atoms (apply $R_1$ once, which creates a null, then apply $R_2$ $n$ times). Hence, there is no finite set of rules $\ruleset'$ such that $chase_\infty(\instance_n,\ruleset)$ injectively maps to $chase_1(\instance,\ruleset')$ for any $n$.
\end{example}

Motivated by this example, we now introduce a new class of rules, called \emph{pieceful}.  

\subsection{The Pieceful Class}
 In short, pieceful rule sets ensure that for any rule application, the entire rule frontier is mapped either to terms from the initial instance or to terms that occur in atoms brought by a single previous rule application. 

\begin{definition}[Pieceful Derivation]
 An $\ruleset$-derivation $(\instance_0 = \instance), \ldots, \instance_k$, is \emph{pieceful} if for all $i$ with $0 < i \leq k$, 
 either $\pi_i(\Fr{R_i}) \subseteq \Term{\instance}$ or there is $j < i$ such that 
$\pi_i(\Fr{R_i}) \subseteq \Term{A_j}$,  
where $A_j = \pi_j^{\mathrm{safe}}(\head{R_j})$.
\end{definition}


\begin{definition}[Pieceful Rule Set]
 A rule set $\ruleset$ is \emph{pieceful} if (for any instance $\instance$) any $\ruleset$-derivation (from $\instance$) is pieceful. 
\end{definition}


\begin{example} Consider again Example \ref{ex-prime}. From $\instance = \{A(a), B(b)\}$, one builds a non-pieceful derivation  
$(\instance_0 = \instance) (R_1, \pi_1, \instance_1) (R_2, \pi_2, \instance_2)$. Indeed, $(R_1, \pi_1)$ produces $\pi_1^\mathrm{safe}(\head{R_1}) = A_1 = \{p(a,\nu_0)\}$, with $\nu_0$ the null created from $z$. Then, $\pi_2 = \{ x \mapsto a, y \mapsto b, z \mapsto  \nu_0\}$. Since $\Fr{R_2} = \{y,z\}$ is mapped to $\{b,\nu_0\}$, with $b \not \in \Term{A_1}$ and $\nu_0 \not \in \Term{\instance}$, this derivation is not pieceful, hence neither is $\ruleset$. 
\end{example}

Interestingly, when a rule set is not pieceful, it is possible to build instances that generalize the situation from Example~\ref{ex-prime}, so that the chase creates  nulls 
that occur in an arbitrarily large  number of atoms, as shown by next Proposition~\ref{prop-not-sgbts-unbounded-arity-of-nulls}. As will become clear later (Proposition~\ref{prop-sgbts-chase-k-pieces}), the reciprocal statement is true when the rule set is chase-finite. 

 \begin{proposition}\label{prop-not-sgbts-unbounded-arity-of-nulls}
 If $\ruleset$ is not pieceful then, for all $n$, there exist an instance $I_n$ and a null $\nu_n$ such that $\nu_n$ occurs in at least  $n$ atoms in $chase_\infty(I_n,\ruleset)$. 
\end{proposition}


\begin{proof} (Sketch)
 Let $\instance^0,\ldots,\instance^{n-1},\instance^{n}$ be a non-pieceful derivation such that $\instance^0,\ldots,\instance^{n-1}$ is pieceful, \emph{i.e.}, $(R_n,\pi_n)$ applied on $\instance^{n-1}$ is the first application that violates the pieceful condition. 
 From that derivation, we build a set of instances $\{\instance_i\}$ such that the chase of $\instance_i$ contains a null that occurs in at least $i$ (distinct) atoms.
 
 At least one frontier variable of $R_n$ is mapped to a null. Let $k$ be the largest integer such that $(R_k, \pi_k)$ introduces a null (in $\instance^k$) to which a frontier variable of $R_n$ is mapped by $\pi_n$. Let  $\nu^*$ be this null. 
 We define by induction on $i$ the instance $\instance_i$ as follows, 
 where each  $f_i(\instance^{k-1})$ denotes a freezing of $\instance^{k-1}$: 
  \emph{(i)} $\instance_0 = f_0(\instance^{k-1})$ and \emph{(ii)} for any $i \geq 1$, $\instance_i = \instance_{i-1} \cup f_i(\instance^{k-1})$ where: if $x \in \pi_k(\Fr{R_k})$, then $f_i(x) = f_0(x)$, otherwise $f_i(x)$ is a fresh constant (w.r.t. the whole construction).
  Intuitively, $\instance_i$ is built from $i+1$ copies of $\instance^{k-1}$, where all the terms have been freshly renamed except for those in $\pi_k(\Fr{R_k})$. From any $\instance_i$, we can build a derivation that mimics the initial derivation $\instance^{k-1}, \ldots,  \instance^{n}$ and show that $\nu$ occurs in $i+1$ atoms.
\end{proof}

It follows that any parallelisable set is pieceful:


\begin{proposition}\label{prop-parallelisable-imply-pieceful}
 If $\ruleset$ is parallelisable, then it is pieceful. 
\end{proposition}

\begin{proof}
Assume that $\ruleset$ is parallelisable by $\ruleset'$.  Notice that for any $I$, any null in $\chase{\infty}{\instance}{\ruleset}$
 occurs in at most $h$ atoms, where $h$ is the maximal size of a rule head in $\ruleset'$ (indeed, two distinct rule applications from $\ruleset'$ cannot share any null). 
 Since $h$ is bounded independently from $I$, by contraposite of Proposition~\ref{prop-not-sgbts-unbounded-arity-of-nulls}, $\ruleset$ is pieceful.
\end{proof}


\begin{figure}
\begin{center}
 \includegraphics[height = 8 cm]{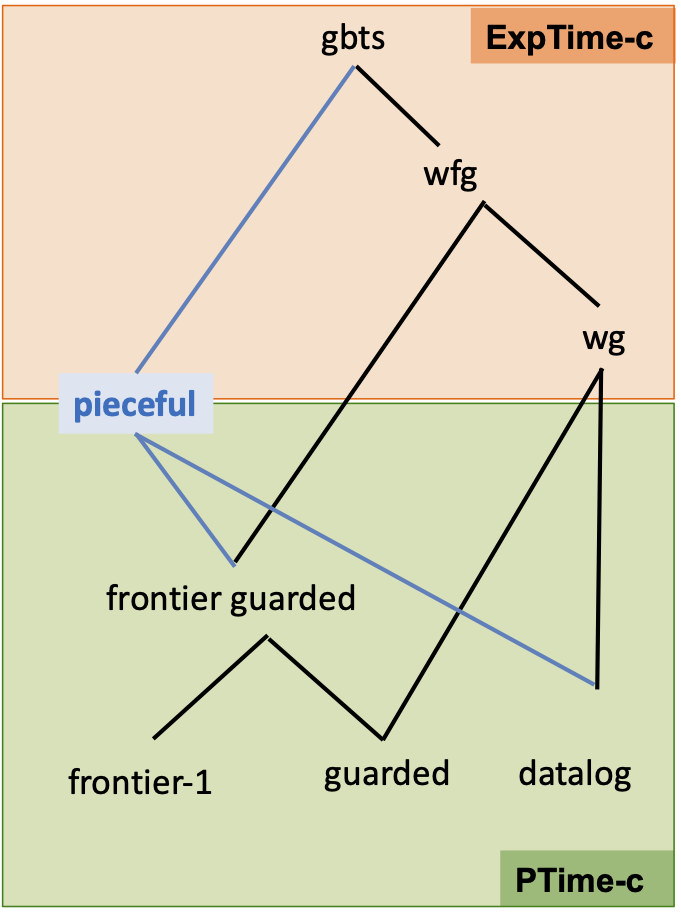} 

\caption{The gbts classes ordered by syntactic inclusion (and the data complexity of the associated Boolean CQ answering problem)}
\label{fig-gbts}     
\end{center}
\end{figure}

How does the pieceful class fit in the existential rule landscape? Clearly, pieceful rule sets are \emph{greedy-bounded-treewidth sets (gbts)} \cite{bmrt:11,rtbm-arxiv:14}\footnote{Indeed, a pieceful derivation is a specific greedy derivation and gbts are those sets for which all derivations are greedy.}, an expressive family for which CQ answering is decidable. The gbts class includes some prominent existential classes, see Figure~\ref{fig-gbts}.  
There are three basic classes: (plain) \emph{datalog} (e.g., \cite{Alice}), in which there are no existential variables at all; \emph{guarded} rules, in which all the variables from a rule body are guarded, \emph{i.e.}, jointly occur in a body atom \cite{Cali2009,cgk:13}; \emph{frontier-one} rules, in which the frontier of a rule is restricted to (at most) one variable \cite{blms:09}. 
Combining guardedness and frontier-based restrictions leads to \emph{frontier-guarded} rules, in which only the frontier of a rule needs to be guarded \cite{blm:10}. The guardedness  condition is further relaxed in \emph{weakly (frontier) guarded (w(f)g)} rules, in which only the (frontier) variables possibly mapped to nulls during the chase need to be guarded. 

As pictured in Figure~\ref{fig-gbts}, the pieceful class includes datalog and frontier-guarded, but not wg (Example \ref{ex-prime}  is wg) and it is not difficult to see that it is actually incomparable with wfg.

\begin{proposition}
Any set of frontier-guarded existential rules or of datalog rules is pieceful.
\end{proposition}


\subsection{Characterization of Parallelisability} \label{sec-charac}
We first point out that the chase of any instance $\instance$ is equal to the union of 
 its \emph{pieces} (w.r.t. nulls). 
A fundamental property of pieceful rules is that the application of a rule $B \rightarrow H$ at level $i$ can \emph{never connect two pieces} of $\chase{i-1}{\instance}{\ruleset}$; when it does not produce an atom already present in $\chase{i-1}{\instance}{\ruleset}$, either it creates a new piece of size $|H|$ (when the frontier is mapped to terms from $\instance$), or it makes an existing piece grow by $|H|$ atoms (when the frontier is mapped to terms with at least one null). 

\begin{proposition}\label{prop-sgbts-chase-k-pieces}
 If $\ruleset$ is pieceful then for any instance $\instance$ and integer $k$, the maximal size of a piece in $\chase{k}{\instance}{\ruleset}$ is bounded independently from $\instance$. 
\end{proposition}

\begin{proof}(Sketch)
Given $\instance$ and $\ruleset$, 
we note $P(i)$ the maximal size of a piece in $\chase{i}{\instance}{\ruleset}$. We prove that  $P(0) = 1$ \mlmm{and, for $i \geq 0, P(i+1) \leq (P(i)\times {a})^{\mathrm{fr}} \times h \times |\ruleset|$}, where $a$ is the maximal arity of a predicate, $\mathrm{fr}$ and $h$ are the maximal size of a rule frontier in $\ruleset$ and a rule head in $\ruleset$, respectively.
\end{proof}

\begin{corollary} 
If $\ruleset$ is pieceful and bounded then for any instance $\instance$, the maximal size of a piece in $\chase{\infty}{\instance}{\ruleset}$ is bounded independently from $\instance$.
\end{corollary}

Based on this corollary, we are now able to show that any bounded pieceful set is parallelisable. We will give another proof of this result in Sect. \ref{sec-composition} (see Cor. \ref{corollary-pieceful-bounded-imply-parallelisable}).

\begin{proposition}
\label{prop-pieceful-bounded-imply-parallelisable}
 If $\ruleset$ is pieceful and bounded then it is parallelisable. 
\end{proposition}
\begin{proof}(Sketch)
 As $\ruleset$ is pieceful and bounded, there is a finite set of pieces $\mathcal{P}$ such that for any instance $\instance$, any piece of $\chase{}{\instance}{\ruleset}$ is isomorphic to a piece of $\mathcal{P}$.  
 By isomorphism from a piece $P_1$ to a piece $P_2$, we mean a bijection $b$ from $\Term{P_1}$ to $\Term{P_2}$ such that, for  all $x \in \Term{P_1}$, $b(x)$ is a null iff $x$ is a null, and $b(P_1) = P_2$. 
 
 For any piece $P \in \mathcal{P}$, let $(\vect{c_P})$ be a tuple obtained by a total ordering on $\Const{P}$, and  consider the query $q_P(\vect{c_P}) = \exists \vect{y}\ P(\vect{c_P},\vect{y})$, 
 with $\vect{y}$ denoting the variables from $P$. Let $Q_P$ be a finite complete set $\ruleset$-rewritings of $q_P$ (note that we need here the general definition of piece-unifiers that deals with constants). There is such set for any $P$, as $\ruleset$ is bounded, hence \emph{fus}.
   We define $\ruleset'$ as the set of rules of the shape
 \(q'_P(\vect{x}) \rightarrow \exists\vect{y} P(\vect{x},\vect{y}),\)
where  $\vect{x}$ is a tuple of variables in bijection with $\vect{c_P}$, for $P\in\mathcal{P}$ and $q'_P\in Q_P$. 
\end{proof}

From Propositions \ref{prop-parallelisable-imply-bounded}, \ref{prop-parallelisable-imply-pieceful} and \ref{prop-pieceful-bounded-imply-parallelisable}, we finally obtain the following characterization of parallelisability: 

\begin{theorem}\label{th-parallelisable}
A rule set is parallelisable if and only if it is both bounded and pieceful.
\end{theorem}



\section{Parallelisability and Rule Composition}\label{sec-composition}


Early work on datalog has shown that a datalog rule set $\ruleset$ is `parallelisable' (according to our definition)
if and only if it is bounded 
(see, \emph{e.g.}, \cite{DBLP:journals/jacm/GaifmanMSV93}). Moreover, such set, say $\ruleset^{\star}$, can be computed by an operation called \emph{unfolding}: given rules $R_1: B_1 \rightarrow H_1$ and $R_2: B_2 \rightarrow H_2$, and a (most general) unifier $u$ of an atom $A$ in $B_2$ with the atom in $H_1$, the unfolding of $R_2$ by $R_1$ is the rule 
$u(B_1) \cup  u(B_2 \setminus \{A\}) \rightarrow u(H_2)$; starting from $\ruleset$, on can build $\ruleset^{\star}$ by repeatedly unfolding a rule from $\ruleset^{\star}$ by a rule from $\ruleset$, until a fixpoint is reached. This is illustrated by the next example. 

\begin{example}[Datalog Unfolding] \label{ex-datalog}
Let $\mathcal R = \{R_1, R_2, R_3\}$ with:\\
$R_1 : A(x) \rightarrow B(x)$
\\
$R_2 : C(x)  \rightarrow D(x)$
\\
$R_3 : B(x) \wedge D(x)  \rightarrow G(x)$
\\
Denoting $R_i \circ R_j$ the unfolding of $R_i$ by $R_j$, we obtain: 
\\
$R_3 \circ R_1 : A(x) \wedge D(x) \rightarrow G(x)$ 
\\
$R_3 \circ R_2: C(x)\wedge B(x) \rightarrow G(x)$ 
\\
$(R_3 \circ R_1) \circ R_2 : A(x) \wedge C(x) \rightarrow G(x)$
\\
$(R_3 \circ R_2) \circ R_1 = (R_3 \circ R_1) \circ R_2$. 
\\
Finally,  $\ruleset^{\star} = \ruleset \cup \{R_3 \circ R_1, R_3 \circ R_2, (R_3 \circ R_1) \circ R_2 \}$. 
Here $\ruleset^{\star} $ is finite, as $\ruleset$ is bounded. 

\end{example}

 In the following, we will consider composition of existential rules, which generalizes this notion of unfolding.

\subsection{Rule Composition}\label{sec-composition}
Composition of existential rules has been exploited as a means of tracing `chained' sequences of rule applications (\emph{e.g.}, \cite{DBLP:conf/ecai/BagetGMR14,DBLP:conf/aaai/WangWZ18}).  It is naturally based on the notions of piece-unifier and rewriting.  The next definition of rule composition furthermore takes pieces into account.

\begin{definition}[Rule Composition]
\label{def:existential-composition}
\mlmm{
Given rules $R_1: B_1 \rightarrow H_1$ and $R_2: B_2 \rightarrow H_2$ on disjoint sets of variables and $\mu = (B'_2, H'_1, u)$ a piece-unifier of $B_2$ with $R_1$,  the \emph{(existential) composition} of $R_2$ with $R_1$ w.r.t. $\mu$ is the following existential rule, denoted by $R_2 \circ_\mu R_1$, or simply $R_2 \circ  R_1$:
\begin{enumerate}
\item If $u(\Fr{R_2}) \cap \exist{R_1} = \emptyset$:
       $$R_2 \circ_\mu R_1 = u(B_1) \cup u(B_2 \setminus B'_2) \rightarrow u(H_2)$$
\item Otherwise:
$$R_2 \circ_\mu R_1 = u(B_1) \cup u(B_2 \setminus B'_2) \rightarrow u(H_1) \cup u(H_2)$$
\end{enumerate}
}

\end{definition}

\mlmm{Note that $u(B_1) \cup u(B_2 \setminus B'_2)$ is exactly the rewriting of $B_2$ w.r.t. $\mu$.}

The first case in the definition is when no frontier variable from $R_2$ is unified with an existential variable from $R_1$, \emph{i.e.}, $u(\Fr{B'_2}) \subseteq \Fr{R_1}$. 
Then, defining $R_2 \circ_\mu R_1$ as in Point 2 would lead to a rule with a two-piece head (resp., $u(H_1)$ and $u(H_2)$), which can be decomposed into two single-piece-head rules. Moreover, the rule $u(B_1) \cup u(B_2 \setminus B'_2) \rightarrow u(H_1)$ is useless
because
every application of this rule can be obtained with an application of $R_1$. 
In the second case, since at least one frontier variable from $B'_2$ is unified with an existential variable from $R_1$, $u(H_1) \cup u(H_2)$ forms a single piece. Therefore, restricting the rule head to $u(H_2)$ would result in a loss of information. 
Note that in both cases, the obtained rule $R_2 \circ R_1$ has a single piece head.
 

\begin{example} \label{ex-composition}
Let $\ruleset$ contain three rules:\\
$R_1: A(x) \rightarrow \exists z~p(x,z)$\\
$R_2: p(x,z) \rightarrow B(z)$\\
$R_3: C(x) \land B(y) \rightarrow r(x,y)$\\
The composed rule $R_3 \circ R_2 = p(x',z) \land C(x) \rightarrow r(x,z)$ illustrates Point 1. Defining $R_3 \circ R_2$ as in Point 2 would lead to the following rule with a two-piece head:\\
$p(x',z) \land C(x) \rightarrow B(z) \land r(x,z)$
\\We can see that  $p(x',z) \land C(x) \rightarrow B(z)$ is useless w.r.t. $R_2$. \\
$R_2 \circ R_1 = A(x) \rightarrow \exists z~p(x,z) \land B(z)$ illustrates Point 2.
\end{example}

%
The set $\ruleset^\star$ includes the original set $\ruleset$ as well as all rules obtained by composition. 

\begin{definition}($\mathcal R^{\star}$)
The set of (existentially-) composed rules associated with $\ruleset$, denoted by $\mathcal R^{\star}$, is the possibly infinite set inductively defined as follows: 
\\(base) $\ruleset \subseteq \ruleset^{\star}$, 
\\ (induction) if $R_i, R_j \in \ruleset^{\star}$ and there is a piece-unifier $\mu$ of $\body{R_i}$ with $R_j$, then 
$R_i \circ_\mu R_j \in \mathcal R^{\star}$.
\end{definition}


$\ruleset^\star$ is sound and complete in the sense that  entailment of Boolean CQs is preserved (with $\ruleset^\star$ applied in a single breadth-first step). Formally: for any rule set $\ruleset$, instance $\instance$ and Boolean CQ $q$, holds $I, \ruleset \models q$ if and only if holds $\chase{1}{\instance}{\ruleset^\star} \models q$.
%
Soundness relies on the fact that all the rules in $\ruleset^\star$ are entailed by $\ruleset$ and completeness 
follows from \cite{DBLP:conf/aaai/WangWZ18}, Prop.~2. 

The next proposition yields a more specific completeness result: if $\ruleset$ is a pieceful rule set, then for any instance $\instance$, each piece of $chase_\infty(\instance,\ruleset)$  can be obtained by applying a rule from $\ruleset^\star$ to $\instance$.

\begin{proposition}
\label{prop-pieceful-derivation-star}
Let $\instance$ be an instance and $\ruleset$ be a set of rules. For any pieceful derivation of length $i$ resulting in $\instance_i$, for any piece $P$ in $\instance_i$, either $P \subseteq \instance$ or there exist a rule $R^* \in \ruleset^{\star}$ and a homomorphism $\pi$ from $\body{R^*}$ to $\instance$ such that $P$ maps to $\pi^{\mathrm{safe}}(\head{R^*})$ by an injective homomorphism. 
\end{proposition}

\begin{proof} (Sketch)
  We prove the result by induction on the length $i$ of the derivation.  
 For $i = 1$, if $P \not \subseteq \instance$, then $P$ has been generated by an application of a rule from $R \in \ruleset \subseteq \ruleset^{\star}$. 
  Assuming that the result holds for any piece $P$ of $\instance_{i-1}$ with 
  \mlmm{$i > 1$}, 
  we show it also holds for any piece $P$ of $\instance_i$. Let \mlmm{$I_0,(R_1,\pi_1,I_1),\ldots,(R_i,\pi_i,I_i)$} be a derivation of length $i$. The body of $R_i$ is mapped by $\pi_i$ to $k$ pieces of $I_{i-1}$, where $k \leq |\Var{\body{R_i}}|$. Let $P$ be the piece created or completed by the application of $R_i$ by $\pi_i$. As the derivation is pieceful, $P$ is either a piece of $I_{i-1}$ to which 
  $\pi_i^{\mathrm{safe}}(\head{R_i})$
  has been added, or a new piece. We build by induction on $k$ a rule $R_P$ that maps by $\pi_P$ to $\instance$ and such that $P$ injectively maps to the result of the application of $R_P$ by $\pi_P$.
 \end{proof}

We conjecture that the previous result actually holds without the pieceful restriction, but its proof would be more intricate.
Thanks to this result, we can refine Proposition \ref{prop-pieceful-bounded-imply-parallelisable}: 

\begin{corollary}\label{corollary-pieceful-bounded-imply-parallelisable}
 If $\ruleset$ is  pieceful and bounded then it is parallelisable by a (finite) subset of $\ruleset^{\star}$.	 
\end{corollary}

\begin{proof}
  Since $\ruleset$ is bounded, the maximum size of a piece is bounded and each piece can be generated by a derivation of length bounded by an integer $n$, where $n$ is independent of the instance. 
\end{proof}

\medskip

Since CQ answering with general existential rules is not decidable, $\ruleset^\star$ can be infinite. The next example shows that $\ruleset^\star$ can be infinite even for $\ruleset$ a \emph{bounded} set of rules, which may seem surprising.  

\begin{example}We consider again the prime example, where $\ruleset = \{R_1, R_2\}$ is bounded. 
Let us build $\mathcal R^{\star}$: 
\\
  $R_1: A(x) \rightarrow \exists z~p(x,z)$
 \\ 
 $R_2 : p(x,z) \wedge B(y) \rightarrow r(z,y)$
\\
$R_2 \circ R_1 : A(x) \wedge B(y) \rightarrow \exists z~p(x,z) \land r(z,y)$
\\
$R_2 \circ (R_2 \circ R_1): 
A(x) \wedge B(y) \wedge B(y_1) \rightarrow$ \\ 
$~~~~~~~~~~~~~~~~~~~~~~~~~~~~~~~~~~~~~~~~~~~~~~~~~~~~~~~~~~~
\exists z~p(x,z) \land r(z,y) \land r(z, y_1)$
\\
\emph{etc.}
At each step, one obtains a new rule by the composition $R_2 \circ R^*$, where $R^*$ is the rule created at the preceding step:  \\
$A(x) \land B(y) \land B(y_1) \ldots B(y_i) \rightarrow $\\
$~~~~~~~~~~~~~~~~~~~~~~~~~~~~~~~~~~~~~~~~
\exists z~p(x,z) \land r(z,y) \land r(z, y_1) \ldots \land r(z,y_i)$
\\
Not only $\mathcal R^{\star}$ is infinite, but \mlmm{no finite subset of it is complete.}

\end{example}

Moreover, the previous example shows that rule compositions of the form $R \circ R^*$, with $R \in \ruleset$ and $R^* \in \mathcal R^{\star}$  are required to achieve completeness, while rule compositions of the form $R^* \circ R$ are sufficient in the datalog fragment. 

\subsection{Existential Stability}
In Sect. \ref{sec-charac}, we have shown that a rule set is parallelisable if and only if it is both bounded and pieceful, with this last notion being defined by the behavior of rules during the chase. The next question we study is whether pieceful rule sets can be characterized by their behavior during rule composition. We first introduce the `existential stability' property (the reason for this name will be explained in Sect.~\ref{sec-succint-composition}).

\begin{definition}[Existential stability] Given rules  $R_1$ and $R_2$, we say that a piece-unifier $\mu = (B'_2, H'_1, u)$ of $\body{R_2}$ with $R_1$ satisfies the \emph{(existential) stability} property if the following holds: \\
$\bullet$ either $ u(\Fr{R_2}) \cap  \exist{R_1} =\emptyset$, 
\\ 
$\bullet$ or $\Fr{R_2} \subseteq \Var{B'_2}$.

This notion is extended to a rule set: $\ruleset$ satisfies the stability property if, for any rules $R_1$ and $R_2$ in $\ruleset$, any piece-unifier of $R_2$ with $R_1$ satisfies the stability property. 
\end{definition}

Informally, the stability property says that when a frontier variable from $R_2$ is unified with an existential variable from $R_1$ then all the frontier variables from $R_2$ are unified. By the next example, we point out that the stability property of a rule set may not be preserved when composed rules are added.


\begin{example} Let $\ruleset$ from Example \ref{ex-composition}. It can be checked that it has the stability property. Now, consider $R_2 \circ R_1$ and $R_3$:
\\
$R_2 \circ R_1: A(x) \rightarrow \exists z~p(x,z) \land B(z)$\\
$R_3: C(x) \land B(y) \rightarrow r(x,y)$\\
Then, $R_3 \circ (R_2 \circ R_1)$ involves 
a piece-unifier that does not satisfy the stability property: $\Fr{R_3} = \{x,y\}$; $y$ is unified with $z \in \exist{R_2 \circ R_1}$ but $x$ is not unified, hence $\ruleset \cup \{R_2 \circ R_1\}$ does not have the stability property.
\end{example}

We say that a rule set $\ruleset$ is \emph{stable at the infinite} if $\mathcal{R}^{\star}$ satisfies the stability property.
Next, we show that pieceful rule sets are exactly those rule sets stable at the infinite. 

\begin{proposition}
\label{prop-pieceful-local-stability}
Any pieceful rule set satisfies the stability property.
\end{proposition}

\begin{proof}
Let $R_1:B_1 \rightarrow H_1$ and $R_2: B_2 \rightarrow H_2$ be two rules of a pieceful rule set and let $\mu = (B'_2,H'_1,u)$  be a piece-unifier of $R_2$ with $R_1$. Assume that $\mu$ does not satisfy the stability property. Let $\instance = u(B_1) \cup u(B_2 \setminus B'_2)$. For simplicity here, we confuse $\instance$ and its freezing, and we assume that safe renamings  are the identity (indeed, we will consider a single application of $R_1$ followed by a single application of $R_2$). 
$R_1$ is applicable on $\instance$ by a homomorphism $h_1$ extending $u_{\mid \Var{B_1}}$ (\emph{i.e.}, for all $x \in \Var{B_1}$, $h_1(x) = u(x)$ if $x \in \Fr{R_1}$, otherwise $h_1(x) = x$). $R_2$ is applicable on the result of this application, \emph{i.e.}, $\instance \cup h_1(H_1)$, by a homomorphism $h_2$ extending  $u_{\mid \Var{B_2}}$.
 (\emph{i.e.}, for all $x \in \Var{B_2}$, $h_2(x) = u(x)$ if $u(x)$ is defined, otherwise $h_2(x) = x$). 
Since $u$ does not satisfy the stability property, for some $x \in \Fr{R_2}$, $u(x) \in \exist{R_1}$ and for another $x' \in \Fr{R_2}$, $u(x')$ is not defined.
Hence, $h_2(x) \not  \in \Term{\instance}$, while  $h_2(x') \not \in\Term{h_1(H_1)}$, which shows that the derivation is not pieceful. 
\end{proof}

 \begin{proposition}
\label{prop-pieceful-global-stability}
 If $\ruleset$ is pieceful, then for any rules $R_1$ and $R_2$ from $\ruleset$, $\ruleset \cup \{R_2 \circ_\mu R_1\}$ is also pieceful.
\end{proposition}

\begin{proof}(Sketch)
 We show the contrapositive. Assume $\ruleset'= \ruleset \cup \{R_2 \circ_\mu R_1\}$ is not pieceful. Let $D' = \instance'_0,\ldots,\instance'_n$ be a pieceful $\ruleset'$-derivation on which $(R_{n+1},\pi_{n+1})$ is applicable in a non-pieceful way. 
 We build a non-pieceful $\ruleset$-derivation $D = \instance_1,\ldots,\instance_{\varphi(n)}$ by induction on $n$, which satisfies the following:
 
 \vspace*{-0.1cm}
\begin{itemize}
  \item For all $1 \leq i \leq n$,  there is an isomorphism $\psi_i$ from $\instance'_i$ to $\instance_{\varphi(i)}$
\item For any set $A_j$ (produced by an application in $D$), there is $A'_k$ (produced by an application in $D'$) s.t. ${A_j} \subseteq \psi_n({A'_k})$.
\end{itemize}

\vspace*{-0.1cm}
Consider now $(R_{n+1},\pi_{n+1})$. Then $\psi_n\circ\pi_{n+1}$ is a homomorphism from $\body{R_{n+1}}$ to $I_{\varphi(n)}$. If $R_{n+1} \in \ruleset$, by hypothesis on $(R_{n+1},\pi_{n+1})$, there is no $A'_k$ s.t. $\pi_{n+1}(\Fr{R_{n+1}})) \subseteq \Term{A'_k}$. By induction hypothesis, for every ${A_j}$ in $D$, there is $A'_{k'}$ s.t. ${A_j} \subseteq \psi_n({A'_{k'}})$. Hence, there is no ${A_j}$ in $D$ s.t. $\psi_n\circ\pi_{n+1}(\Fr{R_{n+1}}) \subseteq \Term{A_j}$, and $\ruleset$ is not pieceful.
  If $R_{n+1} = R_2 \circ_\mu R_1$, the same reasoning applies, noting that, since $\mu$ satisfies the stability property, either $\Fr{R_{n+1}} \subseteq u(\Fr{R_2})$ or $\Fr{R_{n+1}} \subseteq u(\Fr{R_1})$.
 %
%
\end{proof}

From Prop. \ref{prop-pieceful-local-stability} and \ref{prop-pieceful-global-stability}, we obtain that pieceful rule sets are stable at the infinite.  
We now show the converse direction.  

\medskip
\begin{proposition}
\label{prop-sgbts-local-stability}
Any rule set that is stable at the infinite is pieceful.
\end{proposition}
\begin{proof} (Sketch)
We prove the result by contrapositive. Considering the first application $(R_n,\pi_n)$ that violates the pieceful constraint in a derivation, we consider a piece $P$ that contains $\pi(x)$ for some $x \in \Fr{R}$. We consider $R_P$ that generates $P$ (using Proposition \ref{prop-pieceful-derivation-star}), and build a unifier $\mu$ of $R$ with $R_P$ that violates the stability property.
\end{proof}

%
%
%

%

From the three previous propositions, we obtain the desired result:

\begin{theorem}\label{th-pieceful-stable}
A rule set is pieceful if and only if it is stable at the infinite. 
\end{theorem}

Finally, as a corollary of Th. \ref{th-parallelisable} and  \ref{th-pieceful-stable}, we obtain another characterization of parallelisable sets of rules.

\begin{corollary}
A rule set is parallelisable if and only if it is both bounded and stable at the infinite.
\end{corollary}


\subsection{Beyond Existential Composition}\label{sec-succint-composition}

In this section, we question the notion of rule composition and provide preliminary findings. We define another rule composition operation, which seems to fit better with our intuition and is more succinct (hence, its name `compact'), but which goes beyond existential rules. We then show that, for piece-unifiers that satisfy the stability property, existential and compact compositions actually coincide. 

For datalog, we know that $\ruleset^\star$ is finite if and only if $\ruleset$ is bounded. 
To better understand why this is no longer true for existential rules, let us focus on compositions of the form $R_2 \circ R_1$, with $\exist{R_1} \neq \emptyset$. Intuitively, the rules $R_2 \circ R_1$ capture the situations in the chase  where `an application of $R_1$ leads to trigger a new application of $R_2$'. More formally: for any instance $I$, application of $R_1$ to $I$ yielding $I_1$, 
homomorphism $\pi_2$ from $\body{R_2}$ to $I_1$ such that $\pi_2(\body{R_2}) \not \subseteq I$, with $I_2=\alpha(I_1,R_2,\pi_2)$, there is a composed rule $R_2 \circ R_1$  whose application to $I$ yields an instance isomorphic to $I_2$. 
One might be tempted to conclude that the set of all rules of the form $R_2 \circ_{\mu} R_1$ (\emph{i.e.}, for all piece-unifiers $\mu$)  is able to capture `all applications of $R_2$ that use an atom brought by an application of $R_1$', \emph{i.e.},
$\chase{1}{\instance}{\{R_1, R_2 \circ_{\mu} R_1 |  \forall \mu \}}$ would be equivalent to 
$\instance_1 = \chase{1}{\instance}{\{R_1\})) \cup  \{ \alpha(I_1, R_2, \pi_2) |  \pi_2(\body{R_2}) \not \subseteq \instance \}}$

However, this does not hold, as illustrated below. 

\begin{example} Consider again the prime example with $\ruleset = \{R_1, R_2\}$:
\\
$R_1: A(x) \rightarrow \exists z~p(x,z)$
\\
$R_2 : p(x,z) \land B(y) \rightarrow r(z,y)$
\\
$R_2 \circ R_1 : A(x) \land B(y) \rightarrow \exists z~p(x,z) \land r(z,y)$. 
\\
Let $\instance = \{A(a), B(b), B(c) \}$. Then:
$$chase_\infty(\instance,\ruleset) = I \cup \{ p(a,z_0), r(z_0, b), q(z_0, c) \}$$
As it is obtained by one application of $R_1$ which triggers two parallel applications of $R_2$, this also corresponds to the above $(\instance_1 = chase_1(\instance, \{R_1\})) \cup  \{ \alpha(I_1, R_2, \pi_2) |  \pi_2(\body{R_2}) \not \subseteq \instance \}$. One (breadth-first) chase step with $\ruleset \cup \{R_2 \circ R_1\}$ would produce instead: 
$$I \cup \{ p(a,z_0), p(a,z_1), r(z_1, b), p(a,z_2), r(z_2, c) \}$$
Note that this here equal to $chase_1(\instance, \{R_1, R_2 \circ_{\mu} R_1 |  \forall \mu \})$. 
We can see here that two nulls $z_1$ and $z_2$ are created instead of a single one. Obviously, both results are not equivalent. F.i., the Boolean CQ $q() = \exists u~r(u, b) \land r(u, c)$ would be answered positively in the first case, but not in the second. 
\end{example}

In the previous example, the logical formula associated with $R_2 \circ R_1$ is the following: 
$$\forall x \forall y ~(A(x) \land B(y) \rightarrow \exists z~(p(x,z) \land r(z,y))).$$

Instead, we propose to interpret rule composition as:
$$\forall x ~(A(x) \rightarrow \exists z~(p(x,z) \land \forall y(B(y) \rightarrow r(z,y))))$$

%
%

By removing the knowledge entailed by $R_1$, we obtain: 
$$R_2 \bullet R_1 = \forall x \exists z \forall y ~(A(x) \land B(y) \rightarrow p(x,z) \land r(z,y))$$

\begin{definition}[Compact Composition]
 \label{def:compact-composition}
  Let \\
  $R_1 = \forall\vect{x_1}\forall\vect{y_1} ~[~B_1(\vect{x_1},\vect{y_1}) \rightarrow \exists \vect{z_1} ~H_1(\vect{x_1},\vect{z_1})~]$ and \\
 $R_2 = \forall\vect{x_2}\forall\vect{y_2} ~[~B_2(\vect{x_2},\vect{y_2}) \rightarrow \exists \vect{z_2} ~H_2(\vect{x_2},\vect{z_2})~]$.\\
        Let $\mu = (B'_2, H'_1, u)$ be a piece-unifier of $B_2$ with $R_1$. 
        
   The \emph{compact composition} of $R_2$ with $R_1$ w.r.t. $\mu$, denoted by $R_2 \bullet_\mu R_1$ is the following closed formula:
   \begin{align*}
    \forall \vect {x'_1} &	\forall \vect {y_1}\exists \vect {z_1} \forall \vect {x'_2} \forall \vect {y'_2} ~  \\
   & (u(B_1) \land  u(B_2 \setminus B'_2) \rightarrow \exists \vect{z_2}~ (u(H_1) \land  u(H_2)))
   \end{align*}
      where $\vect {x'_1} = u(\vect{x_1})$, $\vect {x'_2} = \vect{x_2} \setminus \Var{B'_2}$, $\vect {y'_2} = \vect{y_2} \setminus \Var{B'_2}$. 
\end{definition}  

Compact composition is more succinct than existential composition, since a single $\bullet$-composed rule may capture an unbounded number of 
$\circ$-composed rules.  However, the resulting formula is generally not an existential rule. We show below that, when the unifiers involved in rule composition satisfy the stability property, the result of compact composition is equivalent to an existential rule (hence the name `existential stability') and it coincides with existential composition  
(for clarity, we consider the general form of existential composition, \emph{i.e.}, ignore Point 1 in Def. \ref{def:existential-composition}). 

\begin{proposition} \label{prop-compact-existential}
When the stability property is satisfied, the compact composition is equivalent to the existential composition, \emph{i.e.}, 
for any piece-unifier $\mu$ of a rule $R_2$ with a rule $R_1$ satisfying the stability property,  $R_2 \circ_{\mu} R_1 \equiv R_2 \bullet_\mu R_1$ 
(where $\equiv$ denotes the logical equivalence and $\circ$ is defined according to Point 2 of Def. \ref{def:existential-composition}). 
\end{proposition}

\begin{proof} 
Note that $R_2 \circ_\mu R_1$ is equivalent to the formula obtained from $R_2 \bullet_\mu R_1$ by passing $\exists \vect {z_1}$ after $\forall \vect {x'_2} \forall \vect {y'_2}$. 
Here, such an inversion of quantifiers can be done without change of semantics if we can group the atoms in $R_2 \bullet_\mu R_1$ into two sets, such that one that contains all the atoms with variables in $\vect {z_1}$ and the other all the atoms with variables in  $\vect {x'_2} \cup \vect {y'_2}$. Since $u(H_2)$ may contain variables from both $\vect {x'_2}$ and $\vect {z_1}$, this is possible if and only if $u(H_2)$ contains no variable from $\vect {x'_2}$ or no variable from $\vect {z_1}$. We check that it is indeed the case when $\mu$ has the stability property: if  $u(\Fr{R_2}) \cap  \exist{R_1} = \emptyset$, then $u(H_2)$ does not contain any variable from $\vect {z_1}$; if $\Fr{R_2} \subseteq \Var{B'_2}$, then $\vect {x'_2}$ is empty, hence $u(H_2)$ does not contain any variable from $ \vect {x'_2}$. 
\end{proof}
 
Finally, one could think of skolemizing existential rules to get (specific) logic-programming rules. Briefly, skolemization consists of replacing each existential variable in a rule head by a fresh functional term over the rule frontier. 
Then, rule composition is based on classical (most general) unifiers.
However, as illustrated below, the composition of two skolemized existential rules may not be a skolemized existential rule. Actually, this composition can be seen as the skolemization of the formula obtained by compact composition, which leads us again beyond the existential rule fragment. 

 \begin{example}[Skolem composition]
Consider the skolemization (noted $sk$) of the rules from the prime example:\\
$sk(R_1): A(x) \rightarrow p(x,f(x))$\\
$sk(R_2) = R_2: p(x,z) \land B(y) \rightarrow r(z,y)$\\
Then, the composition of $R_2$ with $sk(R_1)$ yields the following rule, where $p(x,f(x))$ could be removed:\\
\centerline{$A(x) \land B(y) \rightarrow p(x,f(x)) \land r(f(x),y)$}
This rule is not the skolemization of any existential rule because $f$ does not span over the whole rule frontier $\{x,y\}$. Instead, it is $sk(R_2 \bullet R_1)$. 
\end{example}




\section{Concluding Remarks}\label{sec-conc}


In this paper, we introduce the notion of parallelisability of a set of existential rules and characterize parallelisable rule sets in two different ways.
One characterization relies on the behavior of rules during the chase, which leads to define a new class of existential rules, namely pieceful. Another characterization relies on the behavior of rules during rewriting, which led us to question the notion of rule composition. 
We believe that these results open up many perspectives, which we now outline. 

\vspace*{-0.3cm}
\paragraph{Application to OBDA}
Mature systems, such as OnTop \cite{DBLP:journals/semweb/CalvaneseCKKLRR17} or Mastro \cite{DBLP:journals/semweb/CalvaneseGLLPRRRS11}, are based on lightweight description logics (typically DL-Lite$_\ruleset$  underpinning the W3C language OWL2 QL) and the mapping is GAV (\emph{i.e.}, mapping assertions $q_S(\vect{x}) \rightarrow q_O(\vect{x})$ satisfy $\Var{q_O} \subseteq \vect{x}$). Some other OBDA systems are based on the lightweight ontological language RDFS, e.g., UltraWrap \cite{DBLP:conf/semweb/SequedaAM14}, where RDFS is extended with inverse and transitive properties, and mappings are still GAV, and Obi-Wan \cite{DBLP:journals/pvldb/BuronGMM20}, strictly restricted to RDFS but with GLAV mappings (i.e., $q_O$ is any CQ).
 We believe that the existential rule framework is particularly well suited to the development of OBDA systems in more expressive settings. Indeed, existential rules 
are able to express both ontological knowledge and relational GLAV mappings; GLAV mappings provide increased flexibility compared to GAV mappings, by their ability to invent values, thanks to existential variables. This yields a uniform setting for the whole OBDA specification, thereby facilitating the analysis of the interactions between $\mathcal O$ and $\mathcal M$, which is key to efficiency. Furthermore, a rich family of existential rule dialects achieving various expressivity/tractability tradeoffs are available. 
In this paper, we have taken a first step towards the design of compilation techniques for this setting, by characterizing the notion of parallelisability. Indeed, when the rule set is parallelisable, ontological reasoning can be totally compiled into the mapping. Our results pave the way for the development of query answering techniques exploiting parallelisation and the specificities of existential rule dialects. Another interesting extension of our work would be to investigate combined reasoning \cite{lutz:09} or partitioned reasoning \cite{blms:11}, which would allow us to go beyond parallelisable rule sets. 
\vspace*{-0.3cm}
\paragraph{Deepening Theoretical Foundations} The novel pieceful class is of interest in itself, specially in the context of query answering. It includes both datalog and frontier-guarded, a prominent class of existential rules. Frontier-guarded itself covers some major DL dialects used in query answering (\emph{e.g.}, DL-Lite$_\mathcal{R}$ or $\mathcal{ELHI}$, see \cite{DBLP:journals/ki/Mugnier20} for more details on the relationships between these DLs and the gbts family). We conjecture that CQ answering with pieceful rules is in PTime for data complexity. By all these features, the pieceful class is related to the recently introduced warded class \cite{DBLP:journals/pvldb/BellomariniSG18}, even if both are incomparable, at least from a syntactic viewpoint (our prime example is warded but not pieceful, on the other hand warded does not include frontier-guarded). 

Concerning the notion of parallelisability itself, we based our study on the semi-oblivious chase. In line with this, parallelisability requires an \emph{injective} homomorphism from $\chase{\infty}{\instance}{\ruleset}$ to $\chase{1}{\instance}{\ruleset'}$ (where $\ruleset'$ is the parallelisation of  $\ruleset$). 
Dropping the injectivity requirement, \emph{i.e.}, considering logical equivalence, would lead to a more general notion of parallelisability, which remains to be investigated.

Finally, we have seen that compact rule composition leads us beyond the existential rule fragment. The obtained formulas belong to strictly more expressive logical fragments that have been studied in particular in \cite{DBLP:conf/pods/GottlobPS15}. Whether such fragments could lead to parallelisation techniques is another open issue.

\paragraph{Acknowledgments} This work is partly supported by the ANR project CQFD (ANR-18-CE23-0003).

\onecolumn
\newpage
\twocolumn
\bibliographystyle{kr}
\bibliography{krcr-biblio}

\newpage
\onecolumn
\appendix
\section{Proofs of Section 3}

\paragraph{Proof of Proposition \ref{prop-not-sgbts-unbounded-arity-of-nulls}}

Let us assume that $\ruleset$ is not pieceful. There thus exists a non-pieceful derivation. From that derivation, we build a set of instances $\{\instance_i\}$ such that the chase of $\instance_i$ contains a null in that occurs in at least $i$ atoms (together with at least $i$ terms from $\instance_i$).
 
 Let $\instance^0,\ldots,\instance^{n-1},\instance^{n}$ be a non-pieceful derivation such that $\instance^0,\ldots,\instance^{n-1}$ is pieceful, \emph{i.e.}, $(R_n,\pi_n)$ applied on $\instance^{n-1}$ is the first application that violates the pieceful condition. 
 At least one frontier variable of $R_n$ is mapped to a null. Let $k$ be the largest integer such that $(R_k, \pi_k)$ introduces a null (in $\instance^k$) to which a frontier variable of $R_n$ is mapped by $\pi_n$. Let  $\nu^*$ be this null.
 We define by induction on $i$ the instance $\instance_i$ as follows, 
 where each  $f_i(\instance^{k-1})$ denotes a freezing of $\instance^{k-1}$: 
  \emph{(i)} $\instance_0 = f_0(\instance^{k-1})$ and \emph{(ii)} for any $i \geq 1$, $\instance_i = \instance_{i-1} \cup f_i(\instance^{k-1})$ where:
 \begin{itemize}
  \item if $x \in \pi_k(\Fr{R_k})$, then for all $i$, $f_i(x) = f_0(x)$,
  \item if $x \not \in \pi_k(\Fr{R_k})$, then $f_i(x)$ is a fresh constant (w.r.t. the whole construction).
 \end{itemize}
  
 Intuitively, $\instance_i$ is built from $i+1$ copies of $\instance^{k-1}$, where all the terms have been freshly renamed except for those in $\pi_k(\Fr{R_k})$. 
 We now associate to $\instance^{k-1},(R_k,\pi_k,\instance^k),\ldots,(R_n,\pi_n,\instance^n)$ a derivation $\hat{\instance}_{\varphi(k-1)},\ldots,\hat{\instance}_{\varphi(k)},\ldots$ by induction on $n-k$:
\begin{itemize}
 \item we define $\hat{\instance}_{\varphi_i(k-1)} = \instance_i$; $f_j$ is a homomorphism from $\instance^{k-1}$ to $\hat{\instance}_{\varphi_i(k-1)}$;
 \item let us assume that $\instance_{\varphi_i(\ell)}$ is defined and that the $f_j$ are homomorphisms from $\instance^{\ell}$ to $\instance_{\varphi_i(\ell)}$; then $(R_{\ell+1},f_j\circ\pi_{\ell+1})$ are all applicable to  $\instance_{\varphi_i(\ell)}$, and let us consider $\instance_{\varphi_i(\ell+1)}$ the instance obtained after applying all these rules; note that these triggers may generate all the same atoms, which is in particular the case for those applied on  $\hat{\instance}_{\varphi_i(k-1)}$. Also note that $\nu^* \in \Term{\hat{\instance}_{\varphi_i(k) }}$. We extend $f_j$ by mapping the null $\nu$ of $\instance^{k+1}$ introduced by the variable $z_\nu$ to the null $\nu'$ is null introduced by $z_\nu$ by the application $(R_j,f_0\circ\pi_j)$. 
 \end{itemize}

Let us now consider $\instance_{\varphi_i(n)}$. $(R_n,f_j\circ\pi_n)$ is a trigger on $\instance_{\varphi_i(n)}$ for any $j \leq i$. Moreover, as $\nu^* \in \pi_n(\Fr{R_n})$ and some term of $\Fr{R_n}$ is mapped to a term of $\instance^{k-1}$, all these triggers correspond to different rule applications, which each adds at least one atom where $\nu$ occurs. This concludes the proof.

\paragraph{Proof of Proposition \ref{prop-sgbts-chase-k-pieces}}

We prove by induction that, for any  $\instance$ and $\ruleset$, $P(k)$ is upper bounded by a function of $k$ and parameters depending only on $\ruleset$. 
For $k = 0$, $P(k) = 1$ (since each atom forms its own piece). 
Assume the property is true for $\chase{i}{\instance}{\ruleset}$ with $i \geq 0$. Each rule application creates a new piece (of size $|H|$) or adds $|H|$ atoms to an existing piece (or generates an atom already present). 
Each rule $R$ maps its frontier to a piece $P$ and there are at most $|\Term{P} |^{|\Fr{R}|}$ different ways of mapping this frontier. Moreover, denoting by $a$ the maximal arity of a predicate, $|\Term{P}| \leq |P| \times a$. 
Hence, $P({i+1})$ is less than $(P(i)\times {a})^{fr} \times h \times |\ruleset|$, where $fr$ and $h$ are the maximal size of a rule frontier in $\ruleset$ and a rule head in $\ruleset$, respectively. We conclude by noting that, by induction hypothesis, $P(i)$ is upper bounded by a function of $k$ and parameters depending only on $\ruleset$.

\paragraph{Proof of Proposition \ref{prop-pieceful-bounded-imply-parallelisable}}
 As $\ruleset$ is pieceful and bounded, there is a finite set of pieces $\mathcal{P}$ such that for any instance $\instance$, any piece of $\chase{\instance}{\ruleset}	$ is isomorphic to a piece of $\mathcal{P}$.  
 By isomorphism from a piece $P_1$ to a piece $P_2$, we mean a bijection $b$ from $\Term{P_1}$ to $\Term{P_2}$ such that, for  all $x \in \Term{P_1}$, $b(x)$ is a null iff $x$ is a null, and $b(P_1) = P_2$. 
 
 For any piece $P \in \mathcal{P}$, let $(\vect{c_P})$ be a tuple obtained by a total ordering on $\Const{P}$, and  consider the query $q_P(\vect{c_P}) = \exists \vect{y}\ P(\vect{c_P},\vect{y})$, 
 with $\vect{y}$ denoting the variables from $P$. Let $Q_P$ be a finite complete set $\ruleset$-rewritings of $q_P$ (note that we need here the general definition of piece-unifiers that deals with constants). There is such set for any $P$, as $\ruleset$ is bounded, hence \emph{fus}.
   We define $\ruleset'$ as the set of rules of the shape
 \[q'_P(\vect{x}) \rightarrow \exists\vect{y} P(\vect{x},\vect{y}),\]
 
where  $\vect{x}$ is a tuple of variables in bijection with $\vect{c_P}$, for $P\in\mathcal{P}$ and $q'_P\in Q_P$. 
 We claim that $\ruleset'$ is a parallelisation of $\ruleset$.
 

It remains to prove that the provided $\ruleset'$ is indeed a parallelisation of $\ruleset$. 
\begin{itemize}
  \item $\ruleset'$ is finite, as \emph{(i)} $Q_P$ is finite for any $P$, \emph{(ii)} given $Q_P = Q_P'$ if $P$ is isomorphic to $P'$, by mapping $\vect{c_P}$ to $\vect{c_{P'}}$ and \emph{(iii)} $\mathcal{P}$ is finite;
  \item Let $P_{\instance,\ruleset}$ be a piece of $\chase{\infty}{\instance}{\ruleset}$ and let $P \in \mathcal{P}$ be isomorphic to $P_{\instance,\ruleset}$. As rules of $\ruleset$ do not contain any constant, all rewriting steps that can be performed on 
   $q_P(\vect{c_P})$ can also be performed on $P_{\instance,\ruleset}$ (up to the isomorphism). Hence, by completeness of query rewriting, and since $P_{\instance,\ruleset}$ belongs to $\chase{\infty}{\instance}{\ruleset}$, there is some rule of $\ruleset'$ that is applicable on $\instance$ and that generates $P_{\instance,\ruleset}$. Actually, we just described an injective mapping from the pieces of $\chase{\infty}{\instance}{\ruleset}$ to the set of atoms produced by all applications of the rules in $\ruleset'$ on $\instance$. As distinct rule applications introduce disjoint sets of nulls, this shows that $\chase{\infty}{\instance}{\ruleset}$ maps to $\chase{1}{\instance}{\ruleset}$ by an injective homomorphism. 
  \item Let us assume that a rule $R = q_P'(\vect{x}) \rightarrow \exists \vect{y} P(\vect{x},\vect{y})$ 
  of $\ruleset'$ is applied to $\instance$ through homomorphism $\pi$. By the soundness of $\ruleset'$ (which relies on the soundness of query rewriting), and because rules of $\ruleset$ do not contain any constant, there exists an $\ruleset$-derivation from $\instance$ leading to $\instance'$ to which $\exists\vect{y} P(\pi(\vect{x}),\vect{y})$ homomorphically maps. 
  Hence $\chase{1}{\instance}{\ruleset}$ homomorphically maps to $\chase{\infty}{\instance}{\ruleset}$. 
 \end{itemize}

\section{Proofs of Section 4}

\paragraph{Proof of Proposition \ref{prop-pieceful-derivation-star}}

  We prove the result by induction on the length $i$ of the derivation. For $i = 1$, if $P \not \subseteq \instance$, then $P$ has been generated by an application of a rule from $R \in \ruleset \subseteq \ruleset^{\star}$. Let us now assume the result holds for  any piece $P$ of $\instance_{i-1}$, and let us show it also holds for any piece $P$ of $\instance_i$. Let $I_0,(R_1,\pi_1,I_1),\ldots,(R_i,\pi_i,I_i)$ be a derivation of length $i$. $\pi_i$ maps the body of $R_i$ in $k$ pieces of $I_{i-1}$, where $k$ is smaller than the number of variables in 
  $\body{R_i}$. Let $P$ be the piece created or completed by the application of $R_i$ by $\pi_i$. As the derivation is pieceful, $P$ is either a piece of $I_{i-1}$ to which 
  $\pi_i^{\mathrm{safe}}(\head{R_i})$
  has been added, or a new piece. We build by induction on $k$ a rule $R_P$ that maps by $\pi_P$ to $\instance$ and such that $P$ injectively maps to the result of the application of $R_P$ by $\pi_P$.
  \begin{itemize}
   \item If $R_i$ maps its entire body to $\instance$, then we take $R_P = R_i$ and $\pi_P = \pi_i$.
   \item Otherwise, let us assume the result for any rule that maps its body to at most $j-1$ pieces of $\instance_{i-1}$, and let us show the result when $R_i$ maps its body to $j$ pieces of $\instance_{i-1}$. If there is $x \in\Fr{R_i}$ such that $\pi_i(x)$ is a null, then let $P_f$ be the piece of $\instance_{i-1}$ containing that null. Otherwise, let $P_f$ be any of the $j$ pieces reached by $\pi_i$.
     By the hypothesis of induction, let $R,\pi$ be a rule applicable to $\instance$ which creates a piece to which $P_f$ injectively maps by some homomorphism $\psi$. Let us denote by $s_\pi$ a function from $\Term{\head{R}}$ to itself, which is idempotent, the identity on non-frontier terms, and such that $s_\pi(x) = s_\pi(y) \in \Fr{R}$ if and only if $\pi(x) = \pi(y)$ for $x,y \in \Fr{R}$. Let us consider $R_i\circ_\mu R$, where $\mu$ is defined as: $(\pi_i^{-1}(P_f),\pi^{\mathrm{safe}^{-1}}\circ\psi\circ\pi_i(\pi_i^{-1}(P_f)),u)$, where $u$ is such that $u(x) = s_\pi(\pi^{\mathrm{safe}^{-1}}(\psi(\pi_i(x))))$ if $x \in \Term{\pi_i^{-1}(P_f)}$ and $u(x) = s_\pi(x)$ if $x \in \Term{\head{R}}$.
     This is indeed a piece-unifier, as:
   \begin{itemize}
    \item for all $x \in \Fr{R}$, $u(x) \in \Fr{R}$; indeed, $\pi(x) \in \Term{\instance}$, which is mapped by $\pi^{-1}\circ\psi$ back to the frontier of $R$;
    \item if $x$ is a separating variable of $\pi_i^{-1}(P_f)$, it means that $\pi_i(x)$ is a term of $\instance$, which is mapped to $\Fr{R}$;
    \item $u(\pi_i^{-1}(P_f)) = u(\pi^{\mathrm{safe}^{-1}}\circ\psi\circ\pi_i(\pi_i^{-1}(P_f)))$.
   \end{itemize}
   Let $\pi_{\pi_i,\pi}$ defined as follows:
   \begin{itemize}
   \item $\pi_{\pi_i,\pi}(x) = \pi_i(x)$ if $x \in \Term{\body{R_i}} \setminus \Term{\pi_i^{-1}(P_f)}$
   \item $\pi_{\pi_i,\pi}(x) = \pi(x)$ if $x \in \Term{\body{R}}$
   \end{itemize}
   $\pi_{\pi_i,\pi}$ is a homomorphism from $R_i\circ_\mu R$ to $\instance_{i-1}$, which sends the atoms of $\body{R_i}$ to $j-1$ pieces of $\instance_{i-1}$ (all the piece reached by $\pi_i$ except for $P_f$) or to original atoms. The result of the application is a piece into which the piece of $\instance_i$ containing $P_f$ homomorphically maps. Indeed, one extends $\psi$ by mapping each term introduced in $\instance_i$ by an existential variable $z$ of $R_i$ to the term introduced by existential variable $z$ in $R_i\circ_\mu R$. As such, $R_i\circ_\mu R,\pi_{\pi_i,\pi}$ fulfills the induction assumption, which concludes the proof.
  \end{itemize}

\paragraph{Proof of Proposition \ref{prop-pieceful-global-stability}}

We show the contraposite. Let us assume that $\ruleset'= \ruleset \cup \{R_2 \circ_\mu R_1\}$ is not pieceful. Let $D' = \instance'_0,\ldots,\instance'_n$ be a pieceful $\ruleset'$-derivation on which $(R_{n+1},\pi_{n+1})$ is applicable in a non-pieceful way. 
 We build a non-pieceful $\ruleset$-derivation $D = \instance_1,\ldots,\instance_{\varphi(n)}$ by induction on $n$, which satisfies the following:
\begin{itemize}
  \item For all $1 \leq i \leq n$,  there is an isomorphism $\psi_i$ from $\instance'_i$ to $\instance_{\varphi(i)}$
\item For any set $A_j$ (produced by a rule application in $D$), there exists $A'_k$ (produced by a rule application in $D'$) such that ${A_j} \subseteq \psi_n({A'_k})$.
\end{itemize}

Let us present the construction:
 \begin{itemize}
  \item $\instance_{\varphi(1)} = \instance'_1$, $\psi_1$ being the identity; 
  \item if $(R_i,\pi_i)$ is applied to $\instance'_{i-1}$, and $R_i \not = R_2 \circ_\mu R_1$, then $\instance_{\varphi(i)} = \alpha(\instance_{\varphi(i-1)}, R_i, \psi\circ\pi_i)$.  If that derivation is not pieceful, then $\ruleset$ is not pieceful, and we stop the construction here; otherwise, we extend $\psi_{i-1}$ in the natural way, that is by mapping the null introduced in $\instance'_i$ by the instantiation of an existential variable $z$, to the null introduced in $\instance_{\varphi(i)}$ by $z$; notice that $A_i = \psi_i(A'_i)$;
  \item if $(R_i,\pi_i)$ is applied to $\instance'_{i-1}$, and $R_i = R_2 \circ_\mu R_1$, then consider $\hat{\instance}_{\varphi(i)} = \alpha(\instance_{\varphi(i-1)}, R_1, \psi\circ{\pi_i}_{\mid \Term{\body{R_1}}})$ and $\instance_{\varphi(i)} = \alpha(\hat{\instance}_{\varphi(i)}, R_2,\pi')$ 
  
  defined as:
  \begin{itemize}
   \item if $x \in \Term{\body{R_2}} \cap \Term{\body{R_2 \circ_\mu R_1}}$, then $\pi'(x) = \psi\circ{\pi_i}(x)$;
   \item otherwise, $x$ has been unified by $\mu$ with a term appearing in the head of $R_1$: if $u(x) \in \Fr{R_1}$, we define $\pi'(x) = \pi_i(u(x))$, otherwise we define $\pi'(x)$ as the null introduced by the instantiation of $u(x)$ during the application of $R_1$ through $\pi_i$;
  \end{itemize}
  If one of these two rule applications is not pieceful, we have exhibited a non-pieceful $\ruleset$-derivation, and we stop the construction here. Otherwise, we again extend $\psi_{i-1}$ in the natural way, and we notice that $A_{\varphi(i-1)+1} \subseteq \psi_{i}(A'_i)$ 
  and that $A_{\varphi(i)} \subseteq \psi_{i}(A'_i)$.
 \end{itemize}
 Let us now consider $(R_{n+1},\pi_{n+1})$. $\psi_n\circ\pi_{n+1}$ is a homomorphism from $\body{R_{n+1}}$ to $I_{\varphi(n)}$. If $R_{n+1} \in \ruleset$, by hypothesis on $(R_{n+1},\pi_{n+1})$, there is no $A'_k$ such that $\pi_{n+1}(\Fr{R_{n+1}})) \subseteq \Term{A'_k}$. By induction hypothesis, for every ${A_j}$ in $D$, there is $A'_{k'}$ such that ${A_j} \subseteq \psi_n({A'_{k'}})$. Hence, there is no ${A_j}$ in $D$ such that $\psi_n\circ\pi_{n+1}(\Fr{R_{n+1}}) \subseteq \Term{A_j}$, and $\ruleset$ is not pieceful.
 
 If $R_{n+1} = R_2 \circ_\mu R_1$, the same reasoning apply, noting that since $\mu$ satisfies the stability property, either $\Fr{R_{n+1}} \subseteq u(\Fr{R_2})$ or $\Fr{R_{n+1}} \subseteq u(\Fr{R_1})$.

\subsection*{Proof of Proposition \ref{prop-sgbts-local-stability}}
Let $\ruleset$ be not pieceful. 
Let $D = (\instance_0 = \instance),\instance_1,\ldots,\instance_n$ be a pieceful $\ruleset$-derivation and let $(R: B \rightarrow H,\pi)$ be a trigger for $\instance_n$ that violates the pieceful constraint. 
Let $j$ be the largest integer for which $\instance_j$ introduces a null to which a frontier variable of $R$ is mapped by $\pi$, and let $\nu$  be any such null. Let us consider $\instance_{j-1}$. $\instance_n$ is the result of a pieceful derivation starting from $\instance_{j-1}$. Let us consider the piece $P_\nu$ of $\instance_n$ w.r.t. $\Term{\instance_n} \setminus \Term{\instance_{j-1}}$ containing $\nu$. By Proposition \ref{prop-compact-existential}, there exists a rule $R_{P_\nu} \in \ruleset^\star$ that is applicable through $\pi_P$ on $I_{j-1}$ and generates $P'$ such that $P$ maps injectively to $P'$ through $\psi$. 
Let us consider, similarly as in the proof of Proposition \ref{prop-pieceful-derivation-star}, $\mu = (B_\mu = \pi^{-1}(P_\nu),H_\mu = \pi_P^{\mathrm{safe}^{-1}}\circ\psi\circ\pi(\pi^{-1}(P_\nu)),u)$, 
where $u(x)$ is defined as $s_{\pi_P}(\pi_P^{\mathrm{safe}^{-1}}(\psi(\pi(x))))$ if $x\in \Term{B_\mu}$ and as $s_{\pi_P}(x)$ if $x \in \Term{H_\mu}$.
$\mu$ is then a piece unifier of $\body{R}$ with $R_P$. Indeed:
\begin{itemize}
 \item $B_\mu \subseteq \body{R}$ and $B_\mu \not = \emptyset$, as $\nu$ belongs to $P_\nu$;
 \item a separating variable of $B_\mu$ is sent to a null that belong both to $\instance_{j-1}$ and to $\instance_n$. As such, it is sent by $\pi_P^{\mathrm{safe}^{-1}}\circ\psi$ to a frontier term of $R$; 
 \item $u(B_\mu) = u(H_\mu)$.
\end{itemize}

$\mu$ does not satisfy the stability property, as by definition of $\pi$, there is $x \in \Fr{R}$ such that $x \not \in \Var{B_\mu}$, 	while 
$\pi^{-1}(\nu) \in \Fr{R}$ is unified by $\mu$. Hence $u(\Fr{R}) \not = \emptyset$ and $\Fr{R} \not \subseteq \Var{B_\mu}$, so $\mu$ does not satisfy the stability property. 

\end{document}